\documentclass[pmlr]{jmlr}

\RequirePackage{graphicx}
 \usepackage{booktabs}
\usepackage{longtable}
\usepackage{multirow}
\usepackage{multicol}
\usepackage{colortbl}
\usepackage{xcolor}
\usepackage{subcaption}
\usepackage{appendix}
\usepackage{hyperref}

\usepackage{amsmath,amsfonts,bm}
\usepackage{mathtools}
\usepackage{booktabs}

\newcommand{\lnorm}{\left\lVert}
\newcommand{\rnorm}{\right\rVert}

\usepackage{amsmath}
\usepackage{amssymb}
\usepackage{mathtools}
\usepackage{adjustbox}
\usepackage{rotating}

\RequirePackage{algorithm}
\RequirePackage{algorithmic}

\def\eqref#1{equation~\ref{#1}}

\def\1{\bm{1}}

\def\rve{{\mathbf{e}}}

\def\rvx{{\mathbf{x}}}
\def\rvy{{\mathbf{y}}}
\def\rvz{{\mathbf{z}}}

\DeclareMathAlphabet{\mathsfit}{\encodingdefault}{\sfdefault}{m}{sl}
\SetMathAlphabet{\mathsfit}{bold}{\encodingdefault}{\sfdefault}{bx}{n}

\def\gH{{\mathcal{H}}}

\def\gL{{\mathcal{L}}}

\def\gS{{\mathcal{S}}}

\DeclareMathOperator*{\argmax}{arg\,max}

\newcommand{\appendixheading}[1]{
  \section*{\Large\bfseries #1}
  \addcontentsline{toc}{section}{#1}
}
\definecolor{customcolor}{HTML}{E9D8A6}
\newcolumntype{a}{>{\columncolor{customcolor}}c}
\newcolumntype{b}{>{\columncolor{white}}c}
\makeatletter
\def\set@curr@file#1{\def\@curr@file{#1}} 
\makeatother
\usepackage[load-configurations=version-1]{siunitx} 
\usepackage{lipsum}

\theorembodyfont{\upshape}
\theoremheaderfont{\scshape}
\theorempostheader{:}
\theoremsep{\newline}

\jmlrvolume{[VOLUME \# 252]}
\jmlryear{2024}
\jmlrworkshop{Machine Learning for Healthcare}

\title[Cross-Disease Transferability in Zero-Shot Binary Classification of Chest X-Rays ]{XDT-CXR: Investigating Cross-Disease Transferability in Zero-Shot Binary Classification of Chest X-Rays}

\author{\Name{Umaima Rahman}
       \Email{umaima.rahman@mbzuai.ac.ae}\\ 
       \Name{Abhishek Basu}
       \Email{abhishek.basu@mbzuai.ac.ae}\\
       \Name{Muhammad Uzair Khattak}
       \Email{uzair.khattak@mbzuai.ac.ae}\\
       \addr Department of Computer Vision\\
       Mohamed Bin Zayed University of Artificial Intelligence\\
       Masdar City, Abu Dhabi, UAE.\\ \newline
       \Name{Aniq Ur Rahman}
       \Email{aniq.rahman@eng.ox.ac.uk}\\ 
        \addr Department of Engineering Science\\
        University of Oxford\\
       Oxford, UK.
       } 

\begin{document}

\maketitle

\begin{abstract}
This study explores the concept of cross-disease transferability (XDT) in medical imaging, focusing on the potential of binary classifiers trained on one disease to perform zero-shot classification on another disease affecting the same organ. Utilizing chest X-rays (CXR) as the primary modality, we investigate whether a model trained on one pulmonary disease can make predictions about another novel pulmonary disease, a scenario with significant implications for medical settings with limited data on emerging diseases. The XDT framework leverages the embedding space of a vision encoder, which, through kernel transformation, aids in distinguishing between diseased and non-diseased classes in the latent space. This capability is especially beneficial in resource-limited environments or in regions with low prevalence of certain diseases, where conventional diagnostic practices may fail. However, the XDT framework is currently limited to binary classification, determining only the presence or absence of a disease rather than differentiating among multiple diseases. This limitation underscores the supplementary role of XDT to traditional diagnostic tests in clinical settings. Furthermore, results show that \textsf{XDT-CXR} as a framework is able to make better predictions compared to other zero-shot learning (ZSL) baselines. Our code and pre-trained models are available at \href{https://github.com/rumaima/xdt_cxr}{github.com/rumaima/xdt\_cxr}.
\end{abstract}

\section{Introduction}

Motivated by the fact that diseases on the same organ have visual similarities \citep{ozger2020factors}, we aim to answer the question: \textit{``Can a binary classifier be trained on one disease to perform zero-shot classification on another disease on the same organ using medical images of the same modality?"} We refer to the capability of a model to perform cross-disease zero-shot classification as cross-disease transferability (XDT). In this work, we focus on chest X-rays (CXR) to construct a cross-disease transferable  binary classifier. For example, if we have a model that is trained on pneumonia dataset \citep{kermany2018identifying} but it is able to accurately classify COVID-19 samples \citep{rahman2021exploring} then this will be helpful for cases where we have limited training data on COVID. Furthermore, in the future if we have a novel disease that affects the same organ then we can make predictions leveraging the cross-disease capabilities of that model. Furthermore, XDT will be particularly beneficial for resource-limited settings in under-developed countries with scant laboratory facilities. Also, in low TB burden countries where TB diagnosis may not be as readily discernible on X-rays and therefore would benefit from the application of cross-disease transferability. 
Moreover, the XDT framework can also be applied to train on a smaller labeled dataset and evaluate on a larger unlabeled dataset. The XDT framework heavily relies on the success of the embedding space of the vision encoder which is further transformed through a kernel to enable segregation among classes in the latent space. 

Our XDT framework has a limitation that it can only perform binary classification i.e., it can say whether a medical image is diagnosed with a disease or not and it cannot be used to classify an image among different diseases. This is supported by the fact that visual examination like X-rays and CT-scans especially for lungs is performed in addition to viral/bacterial tests and only serves as a reinforcement. Furthermore, this study highlights several significant insights into the application of our XDT framework, particularly regarding novel disease detection and resource management. One of the most notable findings is the ability to use models trained on one disease, such as pneumonia, to predict another disease, like COVID-19. This cross-disease transferability is particularly valuable for making quick and basic decisions in triage scenarios. For instance, clinicians can use XDT to categorize two breathless patients into different sickness categories, helping to determine which patient is more critically ill based on X-ray results.
The rapid and low-radiation nature of X-rays makes them an efficient tool for initial assessments. This is especially beneficial in rural or resource-limited settings where clinicians must decide which patients need to be sent to better-equipped urban hospitals. By serving as a peripheral screening tool, XDT can assist in managing resources effectively within budget constraints. In terms of pathology testing, machine learning models can aid in answering specific clinical questions, such as whether a patient has COVID-19. While traditional models might identify abnormal respiratory patterns without specifying the exact condition, an effective XDT model can act as a rule-out test. If the model indicates that a patient is healthy, it provides reassurance to discharge them, reducing unnecessary hospital admissions. Additionally, XDT can help identify other pathologies that mimic respiratory diseases. For example, in cases of pulmonary embolism, where a patient feels like they have pneumonia but actually has a blood clot in the lungs, the model can differentiate between the conditions, aiding in accurate diagnosis and appropriate treatment.

\subsection*{Generalizable Insights about Machine Learning in the Context of Healthcare}
This study explores the concept of cross-disease transferability in medical imaging, focusing on the potential of binary classifiers trained on one disease to perform zero-shot classification on another disease affecting the same organ. XDT capability is especially beneficial in resource-limited environments or in regions with low prevalence of certain diseases, where conventional diagnostic practices may fail. XDT can serve as a peripheral screening tool, helping clinicians categorize patients, answer specific diagnostic questions, and rule out diseases, thereby optimizing patient care and resource allocation. Furthermore, while traditional models might identify abnormal respiratory patterns without specifying the exact condition, an effective XDT model can act as a rule-out test. If the model indicates that a patient is healthy, it provides reassurance to discharge them, reducing unnecessary hospital admissions.  Although currently limited to binary classification, determining only the presence or absence of a disease, the XDT framework demonstrates superior performance compared to other zero-shot learning (ZSL) baselines. This underscores its supplementary role to traditional diagnostic tests and highlights its potential to enhance diagnostic accuracy, optimize healthcare delivery, and support clinical decision-making in various settings.

\section{Related Work}
\subsection{Zero Shot Learning}
Zero-Shot Learning (ZSL) is a machine learning technique that addresses the challenge of classifying previously unseen classes during the training process. The key idea behind ZSL is to leverage the semantic information of the seen classes to generalize and transfer knowledge to the unseen classes.
\cite{hayat2021multilabel} addresses a fundamental limitation of supervised learning models in chest X-ray (CXR) classification - their inability to predict unseen disease classes during inference. They propose a multi-label generalized zero-shot learning (\textsf{CXR-ML-GZSL}) network to overcome this. The innovation of \textsf{CXR-ML-GZSL} is its ability to learn a visual representation guided by the input's corresponding semantics extracted from medical text. This allows the model to map visual and semantic modalities to a shared latent space, ensuring relevant labels are ranked higher. Crucially, the network is trained only on seen classes, with no auxiliary data for unseen classes, demonstrating its generalisation capability. Experiments on the NIH Chest X-ray dataset show that \textsf{CXR-ML-GZSL} outperforms baselines in recall, precision, F1 score, and ROC-AUC. \cite{Paul} proposed a novel strategy for generalized zero-shot diagnosis of chest radiographs. They leverage the potential of multi-view semantic embedding, which is a promising yet underexplored approach for zero-shot learning (ZSL). Their design also incorporates a self-training phase to address the issue of noisy labels and improve performance on classes not seen during training. Through rigorous experiments, they demonstrate that their model, trained on a single dataset, can consistently perform well across test datasets from diverse sources, including those with significantly different quality.

\subsection{Deep Learning for Chest X-Rays}
Convolutional Neural Networks (CNNs) have emerged as a popular deep learning approach in the domain of medical image classification. This can be attributed to several key characteristics of CNNs, including their capacity to effectively learn complex features from data using a relatively smaller number of trainable parameters \citep{deepleraning}. Additionally, the weight-sharing mechanism employed in CNNs allows for efficient utilization of the available parameters, leading to improved performance and generalization capabilities \citep{lecun2015deep}, which are crucial in medical applications where labeled data can be scarce.

\cite{Al-Waisy2020} developed the COVID-CheXNet, a hybrid deep learning framework for timely diagnosis of COVID-19 infection using chest X-ray images, in response to the increasing pressure on healthcare systems during the COVID-19 outbreak. The system employed a multi-step approach, first enhancing the contrast of the X-Ray image through contrast-limited adaptive histogram equalization, the noise level was reduced through Butterworth bandpass filtering, and then fusing the results obtained from two pretrained deep learning models, ResNet34 \citep{he2016deep} and High-Resolution Network, using a parallel architecture to provide radiologists with high confidence in discriminating between healthy and COVID-19-infected individuals. \cite{kundu2021pneumonia} developed a computer-aided diagnosis system for automatic pneumonia detection using chest X-ray images. The authors employed deep transfer learning to address the limited availability of data and designed an ensemble of three CNN models, including GoogLeNet \citep{szegedy2015going}, ResNet-18, and DenseNet-121 \citep{huang2017densely}. Their novel approach involved using a weighted average ensemble technique, where the weights assigned to the base learners were determined by fusing the scores of four standard evaluation metrics: precision, recall, F1-score, and the area under the curve, rather than relying on experimental tuning. However, in some instances the ensemble framework failed to produce correct predictions, and because three CNN models are required to train the proposed ensemble, the computation cost is higher than that of the CNN baselines in the literature.

\subsection{Manifold Learning}
Manifold learning is a family of unsupervised techniques for extracting low-dimensional representations from high-dimensional data by exploiting the intrinsic geometry of the data. Prominent manifold learning algorithms include Isometric Feature Mapping (Isomap), Local Linear Embedding (LLE) \citep{roweis2000nonlinear}, and t-Distributed Stochastic Neighbor Embedding (t-SNE)\citep{van2008visualizing}. \cite{souvenir2006image} parametrized cardiopulmonary image sets using Isomap and reorder the images based on these learned parameters. This reordering results in minimal motion between neighboring images, compared to the original temporal ordering. This simplifies the point correspondence problem and allows pairwise deformations to be estimated and extended into global deformation models.

\cite{zhang2005segmentation} and \cite{zhang2006manifold} build upon active contour frameworks for image segmentation, applying them to noisy cardiopulmonary images. They leverage Isomap with domain-specific distance metrics to learn a parametrization capturing the underlying degrees of freedom in the data. The authors then calculate the contours across all images simultaneously, using the learned parameters as additional shape constraints. 

\subsection{Vision Language Models}
The recent advancements in the field of vision-language models (VLMs) have gained significant attention \citep{radford2021learning}, \citep{yao2021filip}. VLMs are pretrained on a vast collection of image-text pairs available on the internet. This extensive pretraining enables the VLMs to learn the underlying relationships between visual and textual information. The pretraining process of VLMs is typically guided by specific vision-language objectives (\citep{pmlr-v139-radford21a}, \citep{yu2022coca}, \citep{yao2021filip}), which facilitate the learning of meaningful image-text correspondences from the large-scale dataset.

Vision-language models like CLIP\citep{radford2021learning} employs a contrastive learning approach between images and text. The objective is to pull the paired image and text representations closer together in the shared embedding space, while pushing apart the representations of unrelated image-text pairs. Through this contrastive training on large-scale image-text data, the pretrained \textsf{CLIP} model is able to capture rich correspondences between visual and linguistic information. As a result, the pretrained \textsf{CLIP} model can directly perform zero-shot predictions on various tasks by matching the embeddings of any given image and text (e.g., class names), without the need for additional fine-tuning. This zero-shot inference performance also depends on the choice of text prompts used to define the class names. The text prompts for classification can be manually designed using prompt engineering \citep{an2023more} or automatically learned using prompt optimization techniques \citep{zhou2022learning, khattak2023maple, khattak2024learning, khattak2023self, shu2022test, abdul2024align, zhou2022conditional}.  

However, when it comes to specialized domains like healthcare, the applicability of general vision-language models faces limitations. Medical image-text datasets are significantly smaller compared to the vast web-scraped image-caption pairs used to train models like CLIP. Additionally, the semantic relationships between medical images and their corresponding reports can be more nuanced, with potential for false negatives cases where images and reports from different patients convey similar clinical information, yet are incorrectly treated as unrelated during training.
To address these challenges, recent research has explored domain-specific adaptations of vision-language models for the medical field. One such approach is \textsf{MedCLIP}  \citep{wang2022MedCLIP} , which builds upon the core contrastive learning framework of \textsf{CLIP} but introduces several key innovations. First, \textsf{MedCLIP} decouples the image and text modalities, allowing it to effectively scale up the usable training data in a combinatorial manner at a low cost. Second, \textsf{MedCLIP} replaces the standard InfoNCE loss with a semantic matching loss based on medical domain knowledge, helping to eliminate the problematic false negatives encountered in previous methods.

\section{Methodology}
A medical image classifier (MIC) architecture is depicted in Fig.~\ref{fig:fig1} wherein a medical image $\rvx$ is passed through an encoder $f$ which generates a latent vector $\rvz \in \mathbb{R}^d$. This latent vector is then passed through a classifier $g$ to get the normalized class prediction vector $\hat{\rvy}$.
\begin{figure}[h!]
    \centering
    \includegraphics[width=0.5\columnwidth]{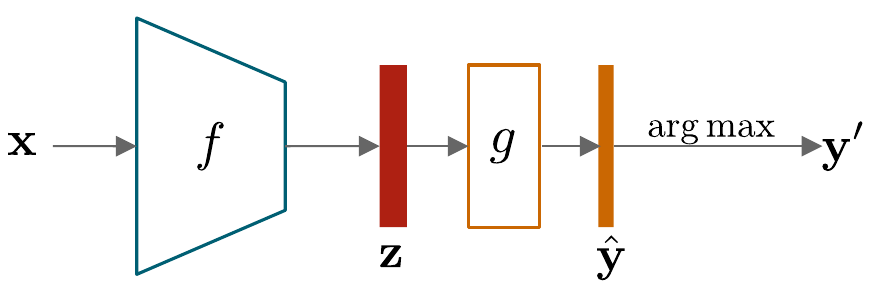}
    \caption{Architecture of the Medical Image Classifier. The encoder $f$ generates a latent vector $\mathbf{z}$, which is passed to the classifier $g$ to produce the normalized class prediction vector $\hat{\mathbf{y}}$. The final output class $\mathbf{y}'$ is obtained via $\arg\max$.}
    \label{fig:fig1}
\end{figure}

Consider a set of pulmonary diseases $\mathfrak{D} \triangleq \{ \mathfrak{d}_1, \mathfrak{d}_2, \cdots \mathfrak{d}_n \}$. The binary-class dataset of disease $\mathfrak{d}_i$ is represented by $\gS_i \triangleq \gS_i^+ \cup \gS_i^-$, where $\gS_i^+$ denotes the set of $\mathfrak{d}_i$-positive samples labelled $\mathfrak{d}_i$, and $\gS_i^-$ is the set of $\mathfrak{d}_i$-negative samples labelled $\mathfrak{d}_i'$. A classifier trained on the dataset of disease $\mathfrak{d}_i$ is denoted by $g_i$.
Then, the probability that a chest X-ray (CXR) image is classified as $\mathfrak{d}_i$-positive is denoted as $P(\mathfrak{d}_i) \triangleq P(g_i \circ f(\rvx) = \mathfrak{d}_i : \forall \rvx \in \gS_i )$. Similarly, we also denote the probability of a $\mathfrak{d}_i$-negative samples as $P(\mathfrak{d}_i')$. Consider a classifier $g$ which instead of focusing on a specific disease, classifies whether a given sample from any binary-class CXR dataset is healthy $\mathsf{H}$ or unhealthy $\mathsf{H}'$ in which a sample is considered healthy if it has none of the diseases in $\mathfrak{D}$.
\begin{align}
    P(\mathsf{H}) = P\left( \bigcap_{i=1}^{n} \mathfrak{d}_i' \right)
    \implies P(\mathsf{H}) \leq P(\mathfrak{d}_i'), \quad \forall i \in [n].
\end{align}
The binary classifier $g_i$ can be viewed as a hyperplane in $\mathbb{R}^d$ segregating the $\mathfrak{d}_i$ positive and negative samples in $\gS_i^+$ and $\gS_i^-$ (see Fig.~\ref{fig:decision_boundaries}). Some negative samples may be positive for some other disease $\mathfrak{d}_j \in \mathfrak{D}, j \neq i$ which can be reflected in the generalised healthy vs. unhealthy classifier $g$ as the relation $P(\mathsf{H}) = P(\mathfrak{d}_i') - \epsilon_i$ which forms the basis of Proposition~\ref{as:hyper}.
\begin{figure}[h!]
    \centering
    \includegraphics[width=0.6\columnwidth]{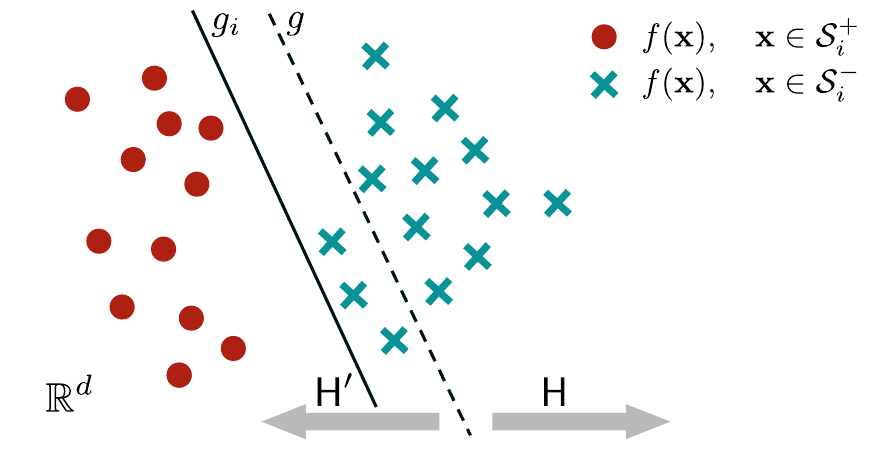}
    \caption{Decision boundaries of disease-specific and healthy/unhealthy binary classifiers. The figure shows two decision boundaries: $g_i$ for a disease-specific classifier and $g$ for a general classifier, separating the feature space $\mathbb{R}^d$ into regions $H'$ and $H$. Samples $f(\mathbf{x}) \in \mathcal{S}_i^+$ (disease-specific) and $f(\mathbf{x}) \in \mathcal{S}_i^-$ (healthy) are represented by different markers.}

    \label{fig:decision_boundaries}
\end{figure}
\begin{proposition}
    $P(\mathsf{H}) = P(\mathfrak{d}_i') - \epsilon_i, \, \exists \epsilon_i \in [0, P(\mathfrak{d}_i')]\,\, \forall i \in [n]$.
    \label{as:hyper}
\end{proposition}

It is possible to find a \textit{decision boundary} outlined by the classifier $g^{\star}$, and an encoding of the  images  $f^{\star}(\rvx) \in \mathbb{R}^d$ such that the \textit{latent representation} of the healthy samples $\mathfrak{d}_i'$ are on one side of the decision boundary, and that of the unhealthy samples $\mathfrak{d}_i$ are on the other. The inter-class segregation as outlined in Fig.~\ref{fig:cluster} will be more feasible with increasing $d \in \mathbb{N}$ as the degrees of freedom for the encoder $f^{\star}$ and the classifier $g^{\star}$ increase. We formalise the above idea in Theorem~\ref{thm:sep}.
\begin{figure}[h!]
\centering
\includegraphics[width=0.7\columnwidth]{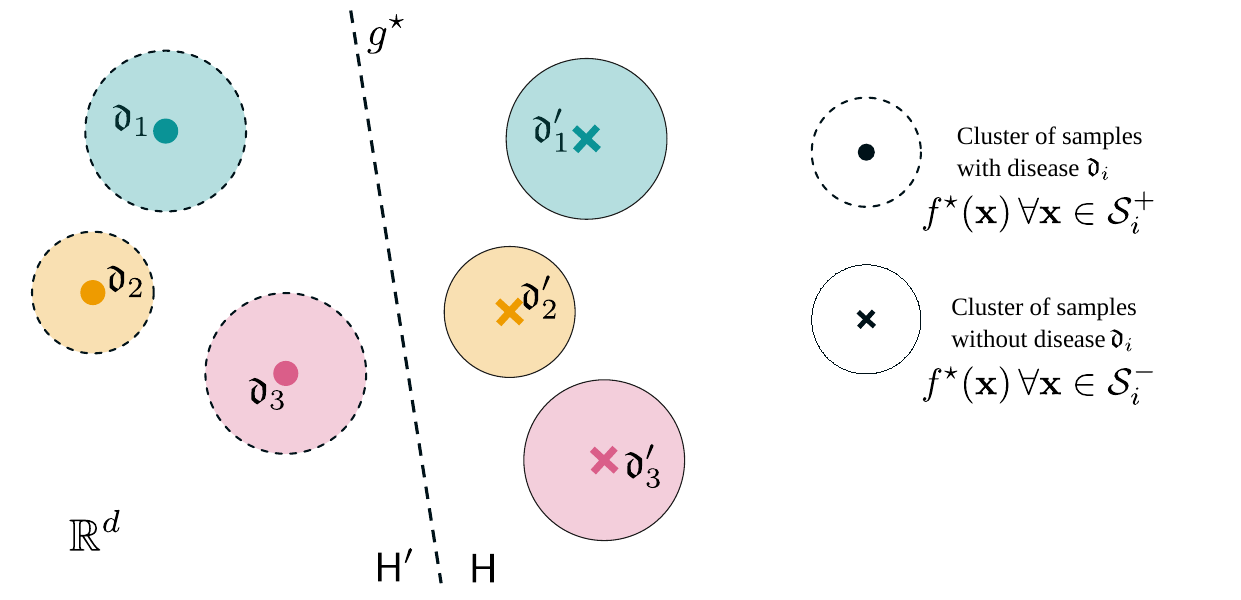}
\caption{Illustration of a general classifier depicting clusters of samples with and without diseases \( \mathfrak{d}_i \) and \( \mathfrak{d}_i' \). The clusters of samples with disease \( \mathfrak{d}_i \) are enclosed in dashed circles, while the clusters of samples without disease \( \mathfrak{d}_i \) are in solid circles. The decision boundary \( g^* \) differentiates between samples with and without diseases. The regions \( H \) and \( H' \) represent the spaces of the samples.}
\label{fig:cluster}
\end{figure}

\begin{theorem}
    For the diseases $\mathfrak{D}$, and datasets $\gS_i \, \forall i \in [n]$, $\exists d \in \mathbb{N}$ for which $\exists f^{\star}, g^{\star}: \forall i \in [n]$ 
    \begin{align*}
        P(\argmax g^{\star} \circ f^{\star}(\rvx) &= \mathfrak{d}_i \mid \rvx \in \gS_i^+) = 1,\\
        P(\argmax g^{\star} \circ f^{\star}(\rvx) &= \mathfrak{d}_i' \mid \rvx \in \gS_i^-) = 1.
    \end{align*}
    \label{thm:sep}
\end{theorem}
\begin{proof}
    From the consequence of Cover's theorem \citep{cover1965geometrical}, it is established that a set of datapoints which is not linearly separable in a given dimension, can be made linearly separable with high probability through a non-linear transformation of the datapoints into a higher dimension. The function $f^{\star}$ which is a neural network acts as the non-linear transformation kernel, which can be affirmed by the universal approximation theorem \citep{guliyev2018approximation} which implies that any function can be approximated by a neural network by having the appropriate weights.
\end{proof}
Contrary to the setting in Theorem~\ref{thm:sep}, we will be learning an encoder $f$, and classifier $g$ through supervised training for a disease $\mathfrak{d}_i \in \mathfrak{D}$, and performing zero-shot evaluation using the resultant model on the diseases $\mathfrak{d}_j \in \mathfrak{D} \setminus \{ \mathfrak{d}_i \}$.
\begin{corollary}
    Given disease $\mathfrak{d}_i \in \mathfrak{D}$, and its corresponding dataset $\gS_i, \, \exists f_i^{\star}, g_i^{\star} : $ 
    \begin{align*}
        P(\argmax g_i^{\star} \circ f_i^{\star}(\rvx) &= \mathfrak{d}_i \mid \rvx \in \gS_i^+) = 1, &\textrm{(Supervised)}\\
        P(\argmax g_i^{\star} \circ f_i^{\star}(\rvx) &= \mathfrak{d}_i' \mid \rvx \in \gS_i^-) = 1, &\textrm{(Supervised)}\\
        P(\argmax g_i^{\star} \circ f_i^{\star}(\rvx) &= \mathfrak{d}_j \mid \rvx \in \gS_j^+) \leq 1, \quad \forall j \in [n] \setminus \{i\}, &\textrm{(Zero-Shot)}\\
        P(\argmax g_i^{\star} \circ f_i^{\star}(\rvx) &= \mathfrak{d}_j' \mid \rvx \in \gS_j^-) \leq 1, \quad \forall j \in [n] \setminus \{i\}. &\textrm{(Zero-Shot)}
    \end{align*}
    \label{cor:zeroshot}
\end{corollary}
\begin{remark}
    Through Corollary~\ref{cor:zeroshot}, we show that given the information of only disease $\mathfrak{d}_i$ dataset, it is still theoretically feasible to construct an encoder-classifier pair $(f_i, g_i)$ which matches the performance of a general encoder-classifier pair $(f^{\star}, g^{\star})$. However, even for $(f^{\star}, g^{\star})$ the performance guarantees do not extend beyond the training data. Moreover, finding the encoder-classifier using one disease dataset and testing it on another disease related to the same organ and modality is an interesting research direction to verify if the model can exhibit cross-disease transferability.
    \label{rmk:xdt}
\end{remark}

\subsection{Problem Formulation}
We aim to design a medical image classifier (like Fig.~\ref{fig:network architecture}) consisting of an encoder-classifier pair $(f, g)$ trained on a labelled dataset $\gS_i = \gS_i^+ \cup \gS_i^-$ of a disease $\mathfrak{d}_i \in \mathfrak{D}$. The learnt model is then used for zero-shot evaluation of the datasets $\gS_j, \forall j \in [n] \setminus \{i\}$. Metrics like classification accuracy, and F1 score can be used to judge the performance of the model.

\subsection{Network Architecture}
\label{subsec:network_arch}
We split the encoder $f$ into two parts through $h \circ f_V$, where $f_V$ is the vision encoder of \textsf{\textsf{MedCLIP}} which remains frozen throughout, and $h$ is a trainable transformer resulting in the latent representation $\rvz \in \mathbb{R}^d$. The latent representation is then passed to a classifier $g$ which returns a normalized class prediction vector $\hat{\rvy}$. The transformer $h$, and the classifier $g$ are trained through a supervised objective, wherein the loss function $\gL$ is passed $\hat{\rvy}$, along with the true label $\rvy$ associated with the sample $\rvx \in \gS$, where $\gS$ is the training data.
\begin{figure}[h!]
    \centering
    \includegraphics[width=0.8\columnwidth]{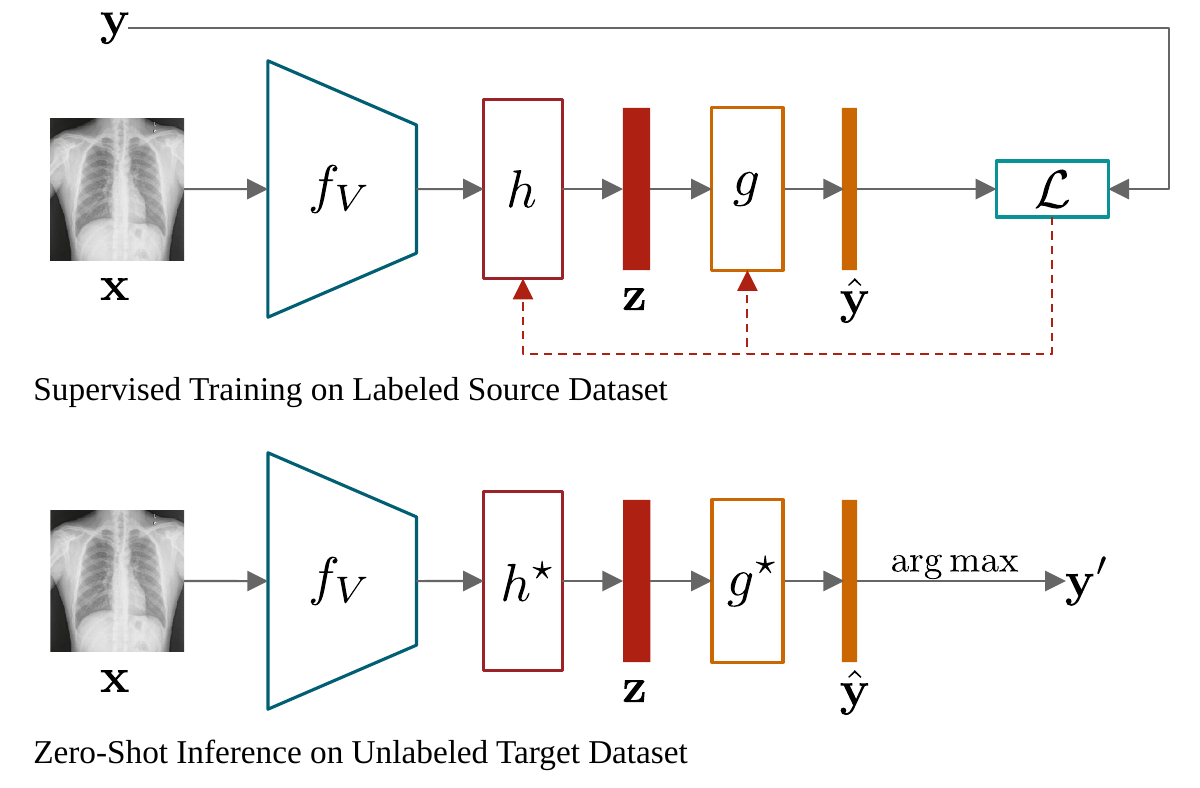}
    \caption{Illustration of the Training and inference networks. The encoder is split into two parts: $h \circ f_V$, where $f_V$ is the frozen vision encoder of \textsf{\textsf{MedCLIP}}, and $h$ is a trainable transformer resulting in the latent representation $\mathbf{z}$. The latent representation $\mathbf{z}$ is passed to a classifier $g$, which returns a normalized class prediction vector $\hat{\mathbf{y}}$. During supervised training, the transformer $h$ and classifier $g$ are trained using a loss function $\mathcal{L}$, which takes $\hat{\mathbf{y}}$ and the true label $\mathbf{y}$. In zero-shot inference, the trained transformer $h^\star$ and classifier $g^\star$ are used to output the class $\arg\max g^\star \circ h^\star \circ f_V(\mathbf{x})$.}

    \label{fig:network architecture}
\end{figure}

The transformer and classifier resulting from the supervised training is dentoed as $h^{\star}$, and $g^{\star}$, respectively. During zero-shot inference, for an input $\rvx$, we get the output class $\argmax g^{\star} \circ h^{\star} \circ f_V(\rvx)$.

\subsection{Loss Function}
We need a loss function that brings the latent representations of samples from the same class closer, and pushes apart samples from different classes which should result in clusters that are far apart. As the gap between the clusters of the two classes increases, it gets easier to find a hyperplane to separate them, which can act as a classifier. Therefore, we make use of a logarithmic contrastive loss defined as follows:

\begin{definition}
The Logarithmic Contrastive Loss (LC) for dataset $\gS$ is defined as:
\begin{align*}
    \gL_{\rm LC}(\gS) \triangleq \frac{1}{\binom{|\gS|}{2}} \sum_{\rvx_i, \rvx_j \in \gS} \underbrace{(2 \rvy_i^\top \rvy_j - 1)\cdot \log\left( \lnorm \rvz_i - \rvz_j \rnorm \right)}_{\rm Clustering} + \lambda \cdot \underbrace{\Big( \gH(\hat{\rvy}_i, \rvy_i ) + \gH(\hat{\rvy}_j, \rvy_j ) \Big)}_{\rm Cross-Entropy}; \, \lambda \in \mathbb{R}^+,
\end{align*}
where $\rvy_i = l(\rvx_i)$ is the one-hot encoded true label, $\rvz_i = h \circ f_V (\rvx_i)$ is the latent representation. The predicted class label is one-hot encoded as $\hat{\rvy}_i = \rve_{\argmax g \circ h \circ f_V( \rvx_i )}$ where $\mathbf{e}_k$ is the $k^{\rm th}$ standard basis, and $\gH(\cdot, \cdot)$ is the cross-entropy function.
\label{def:LC}
\end{definition}
The cross-entropy term ensures that the encoder and classifier learnt during training do not just focus on the clustered appearance in the latent space, but also on the accuracy on the training set. While the Cross-Entropy term is prone to overfitting, the Clustering term acts as a regulariser and generalises the model for unseen diseases, supported by Remark~\ref{rmk:xdt}.

In the clustering term, we write $\log(\lnorm \rvz_i - \rvz_j \rnorm)$ instead of simply $\lnorm \rvz_i - \rvz_j \rnorm$ to ensure that the sum over all the pair of datapoints $\rvx_i , \rvx_j \in \gS$ does not become too high $(\texttt{Inf})$ during computation. The hyperparamter $\lambda \in \mathbb{R}^+$ controls the level of influence the cross-entropy term has on the training objective. We can also replace $\lnorm \rvx_i - \rvz_j \rnorm$ with $\lnorm \rvx_i - \rvz_j  \rnorm + \epsilon$,  $\epsilon \approx 0^+$ to prevent the argument of $\log(\cdot)$ from becoming zero.

In this work, we will also be using cross-entropy loss, and euclidean contrastive loss in out framework to compare against the logarithmic contrastive loss defined above.
\begin{definition}
    The Cross-Entropy Loss (CE) for dataset $\gS$ is defined as:
    \begin{align}
        \gL_{\rm CE}(\gS) \triangleq \frac{1}{|\gS|} \sum_{\rvx_i \in \gS} \gH(\hat{\rvy}_i, \rvy_i).
    \end{align}
\end{definition}

\begin{definition}
    The Euclidean Contrastive Loss (EC) for dataset $\gS$ is defined as:
    \begin{align}
        \gL_{\rm EC}(\gS) \triangleq \frac{1}{\binom{|\gS|}{2}} \sum_{\rvx_i, \rvx_j \in \gS} (2 \rvy_i^\top \rvy_j - 1)\cdot  \lnorm \rvz_i - \rvz_j \rnorm.
    \end{align}
\end{definition}

\section{Experimental Settings}
In this section, we present the design of the experiments conducted to evaluate the performance of our proposed method. The experiments were carefully designed to test the efficacy of our approach across various scenarios and datasets. We begin by detailing the experimental setup, including the datasets used, implementation details and baseline methods for comparison.

\subsection{Datasets}\label{subsec:datasets}
We have used four publicly available Chest X-Ray (CXR) datasets in our experiments:
\paragraph{\textsf{Guangzhou-PN (G)}} Pneumonia dataset published in \cite{kermany2018identifying},  consists of 5,232 pediatric CXR images, which were collected from patients aged one to five years old at Guangzhou Women and Children’s Medical Center in Guangzhou, Guangdong Province, China. There are 1,349 normal and 3,883 Pneumonia x-rays (2,538 x-rays of  Bacterial pneumonia and 1,345 Viral pneumonia).

\paragraph{\textsf{Montgomery-TB (M)}} Montgomery dataset mentioned in \cite{jaeger2014two}, is a collection of 138 posterior-anterior CXR. Which was collected as part of the tuberculosis control program by the Department of Health and Human Services of Montgomery County, MD, USA. The collection includes 80 normal and 58 abnormal x-rays indicating manifestations of tuberculosis. The normal x-rays do not exhibit any signs of tuberculosis or other abnormalities.

\paragraph{\textsf{Shenzhen-TB (S)}} Shenzhen Hospital dataset also in \cite{jaeger2014two}, we will refer to as \textsf{Shenzhen-TB} consists of 662 CXR images as part of the routine care at the No.3 Hospital in Shenzhen, Guangdong Province, China. Of these images 326 normal and 336 abnormal x-rays.

\paragraph{\textsf{Covid (C)}} Covid dataset from \cite{rahman2021exploring} use 1383 test CXR images, this is the largest public COVID positive database. We divided this set into equal number of images in each class, positive (COVID) and negative (non-COVID).

\subsection{Baselines}
\label{baselines}
Our work Cross-Disease Transferability for Chest X-Rays (\textsf{XDT-CXR}) suggests that maximising the cluster gap among the diseases can help in making the model (trained on one dataset) transferable to other datasets. \textsf{XDT-CXR} is based on the vision encoder of \textsf{MedCLIP} \citep{wang2022MedCLIP} and a vision transformer that helps to learn the clusters in the latent space. 
\paragraph{\textsf{Statistical Best}} We consider a theoretical model which consistently predicts the same class for all samples. If the number of samples of one class is more than the other in the dataset in question, the statistical best will always predict the majority label. The Statistical Best model serves an an absolute benchmark as any model which is subpar to it has not learnt anything meaningful from the data.
\paragraph{\textsf{\textsf{MedCLIP}}} \textsf{MedCLIP} \citep{wang2022MedCLIP} is a Vision Language model that trains \textsf{CLIP} \cite{radford2021learning} on medical data. We perform a zero-shot evaluation on our dataset. 
\paragraph{\textsf{\textsf{LaFter}}} This is an approach with an unsupervised tuning of a zero-shot classifier leveraging the cross-modal transferabilities of VLM \cite{mirza2024LaFter}
\paragraph{\textsf{TPT}} Test-Time Prompt Tuning (TPT) \cite{shu2022test}  utilises the embedding space created by a VLM and adapts a prompt at test-time relying on the most confident predictions of multiple augmentations of the test image.
\paragraph{\textsf{\textsf{CXR-ML-GZSL}}} This is a Multi-label Generalised ZSL on the NIH Chest X-Ray dataset which consists of multiple diseases. In this work the model learns to map visual and semantic modalities to a shared latent space, ensuring relevant labels are ranked higher. It is named ML-GZSL due to the reason of it being trained only on seen classes without any auxiliary data for unseen classes. \textsf{\textsf{CXR-ML-GZSL}} \cite{hayat2021multi} uses \texttt{DenseNet} as the vision encoder. For all the other baselines we have used \textsf{MedCLIP}'s vision encoder as the backbone.

\subsection{Metrics}
For our Cross-dataset (XD) zero-shot (ZS) evaluation for binary classification problem, we use accuracy and F1-score to predict the model's capability. 

\paragraph{Accuracy} Accuracy is a metric commonly used in classification tasks to quantify the proportion of correctly classified instances among the total number of instances evaluated. 
\begin{align}
    \texttt{Acc} \triangleq \frac{\rm TP + TN}{\rm TP + TN + FP + FN},
\end{align}
where $\rm TP$ denotes true positives, $\rm TN$ denotes true negatives, $\rm FP$ denotes false positives, and $\rm FN$ denotes false negatives.

\paragraph{F1 score} The F1-score is a harmonic mean of precision and recall, providing a balance between these two metrics. It is particularly useful when dealing with imbalanced datasets.
\begin{align}
    \texttt{F1} \triangleq \frac{2 \cdot \text{precision} \cdot \text{recall}}{\text{precision} + \text{recall}},
\end{align}
where precision is the ratio of true positive predictions to the total number of positive predictions, and recall is the ratio of true positive predictions to the total number of actual positive instances.

\paragraph{Relative Accuracy} The relative change of accuracy with respect to the statistical best accuracy is defined as the relative accuracy 
\begin{align}
    \texttt{Acc'} \triangleq \frac{ \texttt{Acc} - \texttt{Acc}_{\rm SB}}{\texttt{Acc}_{\rm SB}},
\end{align}
where $\texttt{Acc}_{SB}$ is the statistical best accuracy and $\texttt{Acc}$ is the cross-disease zero-shot accuracy of our model.

\subsection{Implementation details}
This section provides an in-depth overview of the implementation details of our proposed method to ensure reproducibility and clarity of our approach.

\subsubsection{Supervised Training on Labeled Source Dataset}
In this stage we input an image which passes through a a frozen vision encoder from \textsf{MedCLIP} and the visual embeddings thus obtained are then fed into a trainable transformer. We used this pipeline to learn class-wise separable features, as theoretically motivated in Section 3. Our model processes visual data using a multi-layer transformer encoder followed by a linear classification layer. Specifically, it consists of a \texttt{nn.TransformerEncoder} with four \texttt{nn.TransformerEncoderLayer} layers, each with an input and output dimension of 512, four attention heads, and a feedforward network dimension of 256. The transformed input features are then passed through a \texttt{nn.Linear} layer, reducing the dimensionality from 512 to 16 for downstream task of classification. The choice of using a transformer compared to other architectures was motivated by the improved accuracy by using transformers as discussed in Appendix \ref{app_choice_architecture}.
During supervised training, both the transformer and the classifier are updated using a loss function that compares the predicted class vector with the true label.


\subsubsection{Zero-Shot Inference on Unlabeled Target Dataset}
For zero-shot inference, the trained transformer and classifier are employed alongside the frozen vision encoder to make predictions on new images. In this phase, the input image is first processed by the frozen vision encoder to extract visual features. These features are then passed through the pre-trained transformer, which transforms them into a latent representation. The classifier takes this latent representation and outputs a prediction vector. The class with the highest score in this prediction vector is selected as the final output. This allows the system to accurately classify images it has not seen during training, leveraging the robust feature extraction of the frozen vision encoder combined with the learned representations from the transformer and classifier.

\subsubsection{Generating Baseline Results}
Baselines using the \textsf{LaFter} and TPT pipelines were established with their official codebases. For \textsf{LaFter}, GPT-3.5 generated textual descriptions for medical dataset classes with specific chest X-ray prompts, which were encoded into embeddings to train a text classifier using cross-entropy loss. The trained classifier generated pseudo labels for self-supervised learning, and \textsf{MedCLIP}'s image encoder replaced CLIP’s, using label smoothing cross-entropy loss. Performance was evaluated on the test set for each dataset in Section 4.1, with Table \ref{tab:zero_shot_results} reporting LaFTer's diagonal values as presented in Table \ref{table:LaFter_performance} of Appendix \ref{app_lafter_cxr}.
We used the original TPT baseline settings, encoding medical text prompts and class labels into textual embeddings, and augmenting each test image into multiple views encoded into visual embeddings with a \textsf{MedCLIP} backbone. Confidence scores were calculated, unreliable ones filtered out, and remaining scores averaged for final classification. The model minimized entropy in the predicted distribution through backpropagation to fine-tune input prompts, accepting or rejecting them based on their effectiveness in reducing classification uncertainty at test-time. These steps were repeated for each dataset in Section 4.1.

\subsection{Model Training}
As discussed in section \ref{subsec:network_arch}, our framework  first performs supervised training on labeled source dataset and then conducts a zero-shot inference on the unlabeled target dataset.
During training, the medical images are passed through the \textsf{MedCLIP} vision encoder and the vision embeddings obtained are send as an input to the transformer. The loss function is used to maximise the cluster distance between the latent  representations learned using the vision transformer. Once the model is trained on one dataset, it is then evaluated on a different dataset. The vision embeddings obtained using a VLM holds valuable information as a result we experimented with \textsf{CLIP} vision encoder with ViT-B/32 backbone as well as \textsf{MedCLIP} vision encoder with Swin Transformer backbone. During exepriments we observed that the results on \textsf{MedCLIP} backbone outperformed the \textsf{CLIP} backbone as discussed in Appendix \ref{app_choice_visual_encoder}. As a result we used \textsf{MedCLIP} backbone for all our experiments. Each dataset was split into a 60-20-20 regime where 60\% of the data was used for training, 20\% for validation and 20\% for testing. The choice of optimizer was critical, and when experiments were conducted using Adam and SGD, we observed that SGD gave better results and hence we used SGD optimiser with a learning rate of $10^{-2}$. The experiments were conducted on a 2nd Gen AMD Epyc processors and on a single Nvidia A100 Tensor Core GPU. 

\section{Results}
In this section, we present the outcomes of the experiments conducted to evaluate the performance of our proposed method \textsf{XDT-CXR}.
\subsection{Supervised Learning and Zero-Shot Image Classification}
In our supervised setting, as discussed in section \ref{subsec:network_arch} that our model was trained in a supervised manner on a dataset $D_i$ and then evaluated on the test set of the same dataset $D_i$. The performance evaluation, detailed in Table \ref{tab:test-results1}, covers four datasets: TB cases from Shenzhen and Montgomery, COVID-19 cases, and pneumonia cases from Guangzhou, with metrics including accuracy (Acc) and F1-score (F1). For the Shenzhen-TB dataset, the model achieved an accuracy of 73.68\% and an F1-score of 0.67, indicating moderate precision and recall. The model performed best on the COVID-19 dataset, with an accuracy of 88.94\% and an F1-score of 0.74, highlighting its effectiveness in identifying COVID-19 cases. Conversely, the performance on the Guangzhou-PN dataset was lower, with an accuracy of 58.50\%, though the F1-score of 0.69 suggests a reasonable balance of precision and recall. The Montgomery-TB dataset results were strong, with an accuracy of 82.76\% and an F1-score of 0.78, reflecting robust performance in TB classification.
\begin{table}[h!]
\caption{Performance evaluation on the test set of the same dataset. Each dataset represents a chest disease, including TB cases from Shenzhen and Montgomery, COVID-19 cases, and pneumonia cases from Guangzhou. The metrics include accuracy (\texttt{Acc}) and F1-score (\texttt{F1}).}
\label{tab:test-results1}
\centering
\small
\begin{tabular}{cccccccc}
\toprule
\multicolumn{2}{c}{\textsf{Shenzhen-TB}} & \multicolumn{2}{c}{\textsf{Covid}} & \multicolumn{2}{c}{\textsf{Guangzhou-PN}} & \multicolumn{2}{c}{\textsf{Montgomery-TB}} \\ \cmidrule(r){1-8} 
  \texttt{Acc} & \texttt{F1} & \texttt{Acc} & \texttt{F1} & \texttt{Acc} & \texttt{F1} & \texttt{Acc} & \texttt{F1} \\ \midrule
 73.68	& 0.67 & 88.94	& 0.74 & 58.50 & 0.69 & 82.76 & 0.78 \\ \bottomrule
\end{tabular}
\end{table}
Furthermore to understand the zero-shot capability, the model trained on dataset $D_i$ is then evaluated on the remaining $D_{n-i}$ datasets in a zero-shot manner. Table \ref{tab:zero_shot_results} gives a comparative analysis of zero-shot evaluation of \textsf{XDT-CXR} on the datasets mentioned in section \ref{subsec:datasets} as well as the baselines described in section \ref{baselines}. The table highlights the exceptional zero-shot performance of XDT-CXR across various medical image classification tasks. XDT-CXR achieves impressive results even when trained on a different dataset than the one being tested. For instance, when trained on the Shenzhen-TB dataset, it shows an accuracy of 80.12\% and an F1-score of 0.65 on the same dataset, while performing notably well on the COVID-19 dataset with an accuracy of 65.41\% and an F1-score of 0.43. Its performance on the Guangzhou-PN dataset, where it achieves an accuracy of 79.31\% and an F1-score of 0.77, and on the Montgomery-TB dataset, where it achieves an accuracy of 71.43\ and an F1-score of 0.56, further underscores its robust generalization capabilities. The varying performance across different datasets also underscores the importance of the vision encoder's embedding space and the kernel transformation process. Effective segregation in the latent space is crucial for the model to distinguish between normal and pathological features across different diseases.
Furthermore, the comparison with other methods, such as \textsf{MedCLIP} and TPT, reveals the competitive nature of the \textsf{XDT-CXR} in the zero-shot scenario, often outperforming these established benchmarks. This demonstrates the potential of the XDT framework to be a valuable tool in settings where rapid adaptation to new diseases is necessary, such as during outbreaks of novel pathogens.

\begin{table}[h]
\caption{Zero-shot evaluation on the test set of a different dataset.}
\label{tab:zero_shot_results}
\centering
\small
\begin{tabular}{lcccccccc}
\toprule
\multicolumn{1}{l}{} & \multicolumn{2}{c}{\textsf{Shenzhen-TB}} & \multicolumn{2}{c}{\textsf{Covid}} & \multicolumn{2}{c}{\textsf{Guangzhou-PN}} & \multicolumn{2}{c}{\textsf{Montgomery-TB}} \\ \cmidrule(r){2-9} 
Method &  \texttt{Acc} & \texttt{F1} & \texttt{Acc} & \texttt{F1} & \texttt{Acc} & \texttt{F1} & \texttt{Acc} & \texttt{F1} \\ \midrule
 \textsf{Statistical Best}              & 57.14 & 0.60 & 73.75 & 0.42 & 62.50 & 0.77 & 51.72 & 0.65 \\ \midrule
 \textsf{{MedCLIP}}                       & 42.11	& 0.57 & 80.84 & 0.63 & 80.93 & 0.86 & 55.17 & 0.61 \\
 \textsf{{LaFter}}                        & 57.14	& 0.00 & 26.25 & 0.42 & 62.5  & 0.77 & 51.72 & 0.00 \\
 \textsf{TPT}                           & 50.76	& \textbf{0.65} & 26.19 & 0.00 & 72.97 & 0.58 & 42.03 & 1.00 \\
 \textsf{{CXR-ML-GZSL}}                   & 57.14	& 0.6  & 73.75 & 0.42& 62.66	& 0.77 & 58.62 & 0.68\\ \midrule
 \textsf{{XDT-CXR} (Shenzhen-TB)}         & x	    & x    & 80.12 & 0.65 & 76.28	& 0.83 & 75.86 & 0.72\\
 \textsf{{XDT-CXR} (Covid)}               & 65.41	& 0.43 & x	   & x    & \textbf{88.14}	& 0.90 & 55.17 & 0.13\\
 \textsf{{XDT-CXR} (Guangzhou-PN)}        & 56.39 & 0.60 & 52.50 & 0.44 & x	    & x    & \textbf{79.31} & 0.77\\
 \textsf{{XDT-CXR} (Montgomery-TB)}       & \textbf{71.43} & 0.56 & \textbf{81.56} & 0.61 & 79.01	& 0.82 & x  &  x\\ \bottomrule
\end{tabular}
\end{table}

Fig. \ref{fig:radar_plot} visualizes the performance of various models across the four datasets: Shenzhen-TB, Covid, Guangzhou-PN, and Montgomery-TB. The models compared are Statistical Best, Zero-shot MedCLIP, LaFter, TPT, CXR-ML-GZSL, and XDT-CXR. Each axis represents one of the datasets, with values indicating the accuracy percentages achieved by each model. The observations include XDT-CXR's superior performance and highest accuracy across all datasets. It outperforms all other models in the Covid dataset and maintains strong performance in the Guangzhou-PN and Montgomery-TB datasets. Statistical Best and CXR-ML-GZSL also perform well, whereas LaFter and TPT show lower accuracy. The plot underscores XDT-CXR's robustness and generalization capabilities, highlighting its superior zero-shot performance in diverse medical imaging tasks.

\begin{figure}[h!]
    \centering
    \includegraphics[width=0.6\linewidth]{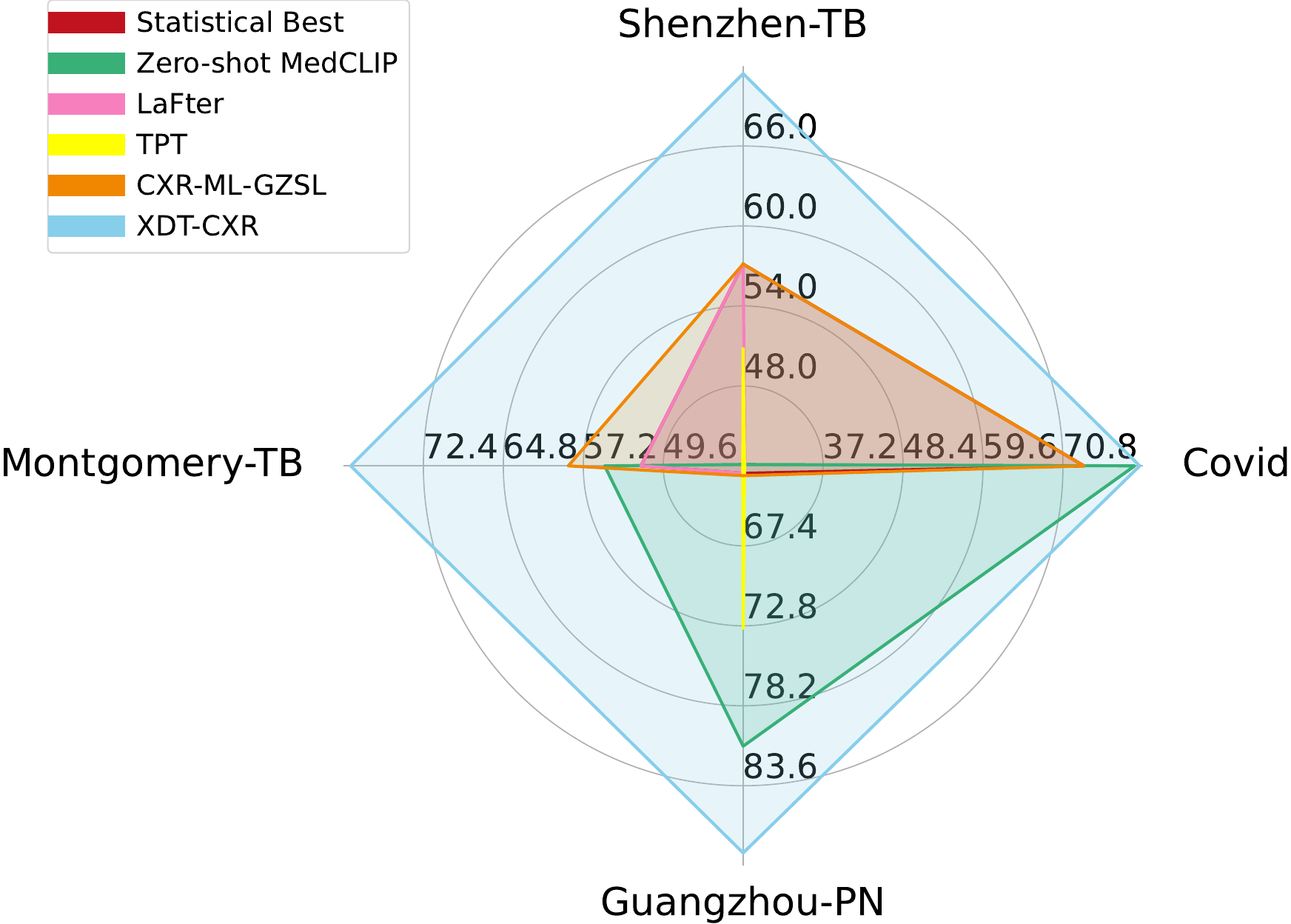}
    \caption{Radar plot showing the accuracy of different baselines across four datasets: Montgomery-TB, Shenzhen-TB, Guangzhou-PN, and Covid. The baselines compared include Statistical Best, Zero-shot \textsf{MedCLIP}, \textsf{LaFter}, TPT, \textsf{CXR-ML-GZSL}, and \textbf{\textsf{XDT-CXR}}.}
    \label{fig:radar_plot}
\end{figure}

\subsection{Impact of Different Loss Functions}
We study the significance of our loss function contributing to the zero-shot performance of our framework \textsf{XDT-CXR}. Table \ref{tab:loss-functions-results} and Table \ref{tab:test-results} highlights the results of the different loss functions compared to the one we used in our training pipeline. The $\gL_{\rm LC}$ loss function seems to facilitate better generalization of learned features across different diseases, which is crucial for zero-shot learning scenarios. This is likely due to its ability to balance between multiple objectives, enhancing the model's capability to identify universal patterns indicative of various chest diseases. The robustness provided by $\mathcal{L}_{\mathrm{LC}}$ indicates its potential as a reliable choice for training diagnostic models that need to operate across varied and unforeseen medical conditions, reflecting real-world clinical needs where rapid adaptation to new diseases is often necessary. Table \ref{tab:loss-functions-results} presents a comparative analysis of various loss functions used in the training pipeline \textsf{XDT-CXR} across the four datasets. The results show that $\gL_{\rm EC}$ generally performs well, with the highest accuracy on Shenzhen-TB (77.44\%) and Guangzhou-PN (72.76\%), though its F1 scores are not the highest across all datasets. $\gL_{\rm CE}$ shows variability in performance, with its best F1 score (0.77) achieved on Guangzhou-PN. On the other hand, $\gL_{\rm LC}$, especially with $\lambda=0.001$, yields consistently high F1 scores and competitive accuracies, notably achieving the highest F1 scores on Covid (0.74) and Montgomery-TB (0.78). Overall, $\gL_{\rm LC}$ with $\lambda=0.001$ appears to offer balanced performance across different datasets, excelling in F1 scores. Further analysis on the choice of a suitable loss function is available in Appendix \ref{app_choice_loss}.

\begin{table}[h!]
\caption{Comparative analysis of different loss functions for the training pipeline \textsf{XDT-CXR} on the test set of the same dataset }
\label{tab:loss-functions-results}
\centering
\small
\begin{tabular}{lcccccccc}
\toprule
\multicolumn{1}{l}{ \textsf{\textsf{XDT-CXR}}} & \multicolumn{2}{c}{\textsf{Shenzhen-TB}} & \multicolumn{2}{c}{\textsf{Covid}} & \multicolumn{2}{c}{\textsf{Guangzhou-PN}} & \multicolumn{2}{c}{\textsf{Montgomery-TB}} \\ \cmidrule(r){2-9} 
  &  \texttt{Acc} & \texttt{F1} & \texttt{Acc} & \texttt{F1} & \texttt{Acc} & \texttt{F1} & \texttt{Acc} & \texttt{F1} \\ \midrule
$\gL_{\rm EC}$              &\textbf{ 77.44}	& 0.66& 77.44	 & 0.65& \textbf{72.76} & 0.82& 75.86 &  0.72\\ 
$\gL_{\rm CE}$              & 42.86	& 0.60& 73.75	 & 0.0 & 62.50 & 0.77& 48.28 &  0.65\\  
$\gL_{\rm LC}, \lambda=0$   & 75.94	& 0.69& 86.12	 & 0.67& 60.26 & 0.71& \textbf{82.76} &  0.78\\ 
$\gL_{\rm LC}, \lambda=0.001 $ & 73.68	& 0.67 & \textbf{88.94}	& 0.74 & 58.50 & 0.69 & \textbf{82.76} & 0.78 \\ \bottomrule
\end{tabular}
\end{table}

\begin{table}[h!]
\caption{Comparative analysis of different loss functions for the training pipeline \textsf{XDT-CXR} on the zero-shot evaluation on the test set of different datasets}
\label{tab:test-results}
\centering
\small
\begin{tabular}{lccccccccc}
\toprule
\multicolumn{1}{l}{\textsf{\textsf{XDT-CXR}}} & Loss & \multicolumn{2}{c}{\textsf{Shenzhen-TB}} & \multicolumn{2}{c}{\textsf{Covid}} & \multicolumn{2}{c}{\textsf{Guangzhou-PN}} & \multicolumn{2}{c}{\textsf{Montgomery-TB}} \\ \cmidrule(r){3-10} 
 Trained on &  & \texttt{Acc} & \texttt{F1} & \texttt{Acc} & \texttt{F1} & \texttt{Acc} & \texttt{F1} & \texttt{Acc} & \texttt{F1} \\ \midrule
 \textsf{Shenzhen-TB}  &$\gL_{\rm EC}$  & x	     & x& \textbf{86.62}	& {0.66}& 37.50	& 0.00& \textbf{79.31}  & {0.77}\\
 \textsf{Covid}   &  & 69.93	 & {0.67}& x	    & x& 73.88	& 0.83& \textbf{79.31}  & {0.77}\\
 \textsf{Guangzhou-PN}  & & 45.11	 & 0.53& 43.82  & 0.41& x	    & x& 72.41  & 0.6\\
 \textsf{Montgomery-TB}  & & \textbf{72.93}	 & {0.67}& 80.69  &{ 0.66}& \textbf{79.17	}& {0.86}& x      &  x\\ 
 \midrule
 \textsf{Shenzhen-TB}  &$\gL_{\rm CE}$ & x	     & x& \textbf{26.25	} & {0.42}& \textbf{62.5}	 & {0.77}& \textbf{48.28}   & {0.65}\\
 \textsf{Covid}   && \textbf{57.14}	 & {0.66}& x	     & x& 37.5	 & 0.00& 51.72   & 0.00\\
 \textsf{Guangzhou-PN}  & & 42.86	 & 0.60& \textbf{26.25}	 & {0.42}& x	     & x& 48.28   & {0.65}\\
 \textsf{Montgomery-TB}  & & 42.86	 & 0.60& \textbf{26.25}	 & {0.42}& 62.5	 & 0.77& x       &  x\\ 
\midrule
 \textsf{Shenzhen-TB} &$\gL_{\rm LC}$ & x	     & x& 80.41 & 0.64& 75.32 & 0.82& 75.86 & 0.72\\
 \textsf{Covid}  &$\lambda=0$ & \textbf{73.68}  & {0.58}& 	x  & x& \textbf{80.93} & 0.83& \textbf{82.76} & 0.78\\
 \textsf{Guangzhou-PN} & & 42.86  & 0.53& 35.00 & 0.37& x	 & x& 72.41 & 0.73\\
 \textsf{Montgomery-TB} & & 71.43  & 0.56& \textbf{81.13} & 0.60& 79.01 & 0.83& x     &  x\\ 
 \midrule
 \textsf{Shenzhen-TB}  &$\gL_{\rm LC}$      & x	    & x    & 80.12 & {0.65} & 76.28	& 0.83 & 75.86 & 0.72\\
 \textsf{Covid}   &$\lambda=0.001$      & 65.41	& 0.43 & x	   & x    & \textbf{88.1}4	& 0.90 & 55.17 & 0.13\\
 \textsf{Guangzhou-PN}  &      & 56.39 & 0.60 & 52.50 & 0.44 & x	    & x    & \textbf{79.31} & 0.77\\
 \textsf{Montgomery-TB}  &   & \textbf{71.43} & {0.56} & \textbf{81.56} & 0.61 & 79.01	& 0.82 & x  &  x\\ \bottomrule
\end{tabular}
\end{table}

\subsection{Impact of Training Size on Zero-shot Capability}
In Fig. \ref{fig:transfer_ratio}, we present the relative accuracy vs. size ratio plot. The relative accuracy is defined as the accuracy of the model with respect to the statistical best (majority label memorisation), and the size ratio is the ratio of the test set size to the train set size. Therefore, a size ratio $< 1$ means that the model was trained on more data than it was tested on, and a size ratio $> 1$ means that the model was trained on a small dataset and tested on a larger one. We have depicted size ratio $< 1$ with red circles, and size ratio $> 1$ with green squares in the scatter plot. The abbreviations S, C, G, and M denote Shenzhen-TB, Covid, Guangzhou-PN, and Montgomery-TB datasets, respectively. The plot compares the performance of models when trained on one dataset and tested on another, highlighting how changes in train/test set sizes impact model accuracy. It shows that models generally achieve higher accuracy when the train set is larger than the test set. This suggests that the size and diversity of the training data are crucial for improving model performance in cross-dataset evaluations.

\begin{figure}[h!]
    \centering
    \includegraphics[width=0.8\linewidth]{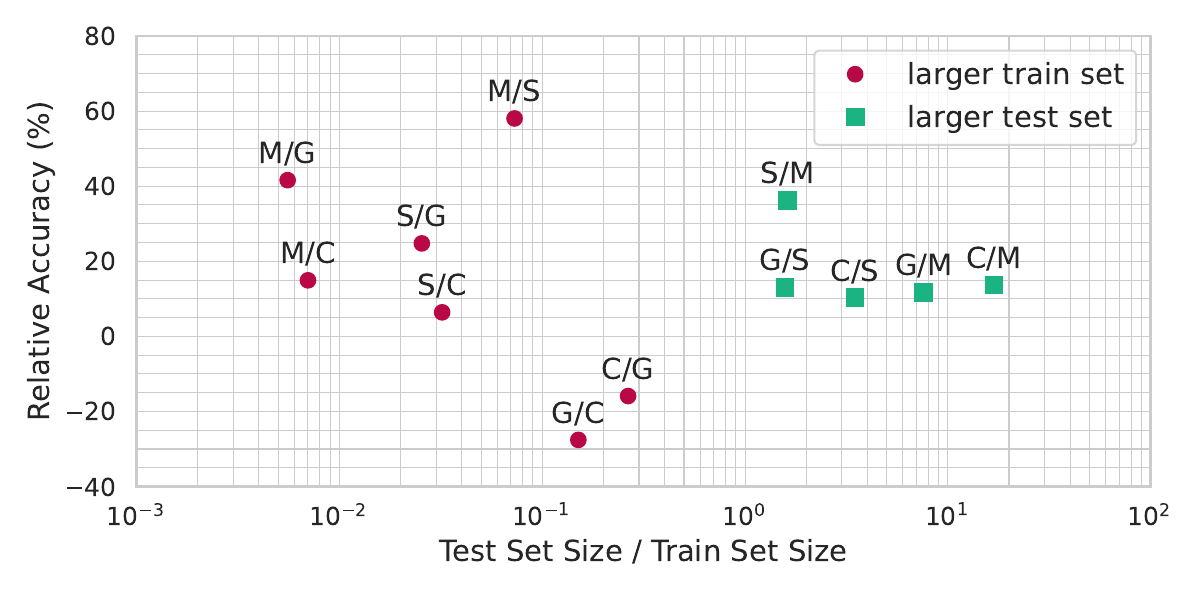}
    \caption{Scatter plot showing the relative accuracy vs. size ratio. A size ratio $< 1$ means that the model was trained on more data than it was tested on, and a size ratio $> 1$ means that the model was trained on a small dataset and tested on a larger one. }
    \label{fig:transfer_ratio}
\end{figure}


\section{Discussion}

The results presented in Tables \ref{tab:zero_shot_results} and \ref{tab:test-results}  provide important insights into the effectiveness of the Cross-Disease Transferability (XDT) framework when applied to chest X-ray (CXR) analysis. In the \textbf{supervised learning} scenario, the models demonstrate varied performance across different diseases and datasets. For instance, the model achieved an accuracy of 88.94\% and an F1-score of 0.74 for COVID-19, indicating a strong capability to identify this disease under controlled, disease-specific training conditions. In contrast, the model's performance on pneumonia from Guangzhou shows a notable drop in accuracy to 58.50\%, suggesting that certain visual disease markers may not be as distinctly captured or may overlap with other conditions, complicating the model's learning process.
In the \textbf{zero-shot classification} scenario, where the models trained on one dataset are tested on others without additional training, the results highlight the potential and challenges of applying the XDT framework to real-world diagnostic tasks. Notably, the \textsf{XDT-CXR} model trained on COVID-19 data performed exceptionally well when applied to the Guangzhou-PN dataset, achieving the highest accuracy of 88.14\% and an F1-score of 0.90. This suggests a significant overlap in visual features relevant to the model between these two diseases, or a robustness of the learned features that generalize well across different conditions within the same organ.
However, other cross-dataset evaluations, such as the \textsf{XDT-CXR} model trained on the Shenzhen-TB dataset evaluated on COVID-19, showed lower performance (accuracy of 80.12\% and F1-score of 0.65) compared to the supervised setting. This indicates the inherent challenge in zero-shot learning where the model must generalize to entirely unseen conditions without losing specificity or sensitivity.

Another noteworthy factor in our experimental setting was that the Guangzhou-PN dataset consists of paediatric data, whereas the other datasets include data from adult patients. In our experiments we investigated two-way cross-disease transferability, i.e., our results show how a model trained on the Guangzhou-PN (G) dataset performs on other datasets in a ZSL setting as well as how other datasets such as Covid (C), Shenzhen-TB (S) or Montgomery-TB (M) which consists of adult CXRs are transferable in giving predictions on Guangzhou-PN datasets as presented in Table \ref{tab:zero_shot_results}.

Overall, the findings suggest that while the XDT framework holds significant promise for enhancing diagnostic efficiency using limited training data, extensive testing and refinement are required to maximize its applicability and reliability in diverse clinical settings.

\paragraph{Limitations.} The XDT framework has a limitation that it can only perform binary classification i.e., it can say whether a medical image is diagnosed with a disease or not. However, it cannot be used to classify an image among different diseases. This is supported by the fact that visual examination like X-rays and CT-scans especially for lungs is performed in addition to viral/bacterial tests and only serves as a reinforcement. Moreover, despite the promising results as discussed above, the limitations in binary classification capability, where the model can only determine the presence or absence of disease rather than identifying specific conditions, remain a critical challenge. Furtehrmore, we understand the limitation that models trained on lung disease might not very well be transferable to a renal disease or a brain disease, given that it affects different organs. However, cross-modal transferability, i.e., checking if (disease A, modality A) is transferable to (disease B, modality B), is an exciting line of enquiry for the future. Few works in the literature perform cross-modal transferability, such as \cite{gu2023cross}, which can motivate the work.
This limitation points to the necessity for further research into multi-class models within the XDT framework that could handle more complex diagnostic tasks.

\section{Conclusion}
In conclusion, this study investigates cross-disease transferability (XDT) in medical imaging, showing that binary classifiers trained on one disease can perform zero-shot classification on another affecting the same organ. Using chest X-rays (CXR), we demonstrate that models trained on one pulmonary disease can predict others, which is useful when data on new diseases is scarce. The XDT framework employs vision encoder embeddings and kernel transformation to differentiate between diseased and non-diseased states, making it especially beneficial in resource-poor settings. However, it currently only supports binary classification, identifying disease presence or absence without distinguishing between multiple diseases. Despite this, the \textsf{XDT-CXR} framework surpasses other zero-shot learning baselines, highlighting its potential as a valuable diagnostic tool.

\acks{We would like to acknowledge the  discussions with Dr. Jennifer Kate van Heerden from the Nuffield Department of Surgical Sciences, University of Oxford, Oxford. Her insights and expertise significantly contributed to the development and refinement of this work.}

\bibliography{sample}

\begin{thebibliography}{36}
\providecommand{\natexlab}[1]{#1}
\providecommand{\url}[1]{\texttt{#1}}
\expandafter\ifx\csname urlstyle\endcsname\relax
  \providecommand{\doi}[1]{doi: #1}\else
  \providecommand{\doi}{doi: \begingroup \urlstyle{rm}\Url}\fi

\bibitem[Abdul~Samadh et~al.(2024)Abdul~Samadh, Gani, Hussein, Khattak, Naseer, Shahbaz~Khan, and Khan]{abdul2024align}
Jameel Abdul~Samadh, Mohammad~Hanan Gani, Noor Hussein, Muhammad~Uzair Khattak, Muhammad~Muzammal Naseer, Fahad Shahbaz~Khan, and Salman~H Khan.
\newblock Align your prompts: Test-time prompting with distribution alignment for zero-shot generalization.
\newblock \emph{Advances in Neural Information Processing Systems}, 36, 2024.

\bibitem[Al-Waisy et~al.(2023)Al-Waisy, Al-Fahdawi, and Mohammed]{Al-Waisy2020}
A.S. Al-Waisy, S.~Al-Fahdawi, and M.A. Mohammed.
\newblock Covid-chexnet: hybrid deep learning framework for identifying covid-19 virus in chest x-rays images.
\newblock \emph{Soft Computing}, 27:\penalty0 2657--2672, 2023.
\newblock \doi{https://doi.org/10.1007/s00500-020-05424-3}.

\bibitem[An et~al.(2023)An, Zhu, Panaitescu-Liess, Mummadi, and Huang]{an2023more}
Bang An, Sicheng Zhu, Michael-Andrei Panaitescu-Liess, Chaithanya~Kumar Mummadi, and Furong Huang.
\newblock More context, less distraction: Improving zero-shot inference of clip by inferring and describing spurious features.
\newblock In \emph{Workshop on Efficient Systems for Foundation Models@ ICML2023}, 2023.

\bibitem[Cover(1965)]{cover1965geometrical}
Thomas~M Cover.
\newblock Geometrical and statistical properties of systems of linear inequalities with applications in pattern recognition.
\newblock \emph{IEEE transactions on electronic computers}, \penalty0 (3):\penalty0 326--334, 1965.

\bibitem[Gu et~al.(2023)Gu, Zhang, Zeng, Zhong, Wang, Liang, Li, and Hu]{gu2023cross}
Xianfan Gu, Yu~Zhang, Wen Zeng, Sihua Zhong, Haining Wang, Dong Liang, Zhenlin Li, and Zhanli Hu.
\newblock Cross-modality image translation: Ct image synthesis of mr brain images using multi generative network with perceptual supervision.
\newblock \emph{Computer Methods and Programs in Biomedicine}, 237:\penalty0 107571, 2023.

\bibitem[Guliyev and Ismailov(2018)]{guliyev2018approximation}
Namig~J Guliyev and Vugar~E Ismailov.
\newblock On the approximation by single hidden layer feedforward neural networks with fixed weights.
\newblock \emph{Neural Networks}, 98:\penalty0 296--304, 2018.

\bibitem[Hayat et~al.(2021{\natexlab{a}})Hayat, Lashen, and Shamout]{hayat2021multi}
Nasir Hayat, Hazem Lashen, and Farah~E Shamout.
\newblock Multi-label generalized zero shot learning for the classification of disease in chest radiographs.
\newblock In \emph{Machine learning for healthcare conference}, pages 461--477. PMLR, 2021{\natexlab{a}}.

\bibitem[Hayat et~al.(2021{\natexlab{b}})Hayat, Lashen, and Shamout]{hayat2021multilabel}
Nasir Hayat, Hazem Lashen, and Farah~E. Shamout.
\newblock Multi-label generalized zero shot learning for the classification of disease in chest radiographs, 2021{\natexlab{b}}.

\bibitem[He et~al.(2016)He, Zhang, Ren, and Sun]{he2016deep}
Kaiming He, Xiangyu Zhang, Shaoqing Ren, and Jian Sun.
\newblock Deep residual learning for image recognition.
\newblock In \emph{Proceedings of the IEEE conference on computer vision and pattern recognition}, pages 770--778, 2016.

\bibitem[Huang et~al.(2017)Huang, Liu, Van Der~Maaten, and Weinberger]{huang2017densely}
Gao Huang, Zhuang Liu, Laurens Van Der~Maaten, and Kilian~Q Weinberger.
\newblock Densely connected convolutional networks.
\newblock In \emph{Proceedings of the IEEE conference on computer vision and pattern recognition}, pages 4700--4708, 2017.

\bibitem[Jaeger et~al.(2014)Jaeger, Candemir, Antani, W{\'a}ng, Lu, and Thoma]{jaeger2014two}
Stefan Jaeger, Sema Candemir, Sameer Antani, Y{\`\i}-Xi{\'a}ng~J W{\'a}ng, Pu-Xuan Lu, and George Thoma.
\newblock Two public chest x-ray datasets for computer-aided screening of pulmonary diseases.
\newblock \emph{Quantitative imaging in medicine and surgery}, 4\penalty0 (6):\penalty0 475, 2014.

\bibitem[Kermany et~al.(2018)Kermany, Goldbaum, Cai, Valentim, Liang, Baxter, McKeown, Yang, Wu, Yan, et~al.]{kermany2018identifying}
Daniel~S Kermany, Michael Goldbaum, Wenjia Cai, Carolina~CS Valentim, Huiying Liang, Sally~L Baxter, Alex McKeown, Ge~Yang, Xiaokang Wu, Fangbing Yan, et~al.
\newblock Identifying medical diagnoses and treatable diseases by image-based deep learning.
\newblock \emph{cell}, 172\penalty0 (5):\penalty0 1122--1131, 2018.

\bibitem[Khattak et~al.(2023{\natexlab{a}})Khattak, Rasheed, Maaz, Khan, and Khan]{khattak2023maple}
Muhammad~Uzair Khattak, Hanoona Rasheed, Muhammad Maaz, Salman Khan, and Fahad~Shahbaz Khan.
\newblock Maple: Multi-modal prompt learning.
\newblock In \emph{Proceedings of the IEEE/CVF Conference on Computer Vision and Pattern Recognition}, pages 19113--19122, 2023{\natexlab{a}}.

\bibitem[Khattak et~al.(2023{\natexlab{b}})Khattak, Wasim, Naseer, Khan, Yang, and Khan]{khattak2023self}
Muhammad~Uzair Khattak, Syed~Talal Wasim, Muzammal Naseer, Salman Khan, Ming-Hsuan Yang, and Fahad~Shahbaz Khan.
\newblock Self-regulating prompts: Foundational model adaptation without forgetting.
\newblock In \emph{Proceedings of the IEEE/CVF International Conference on Computer Vision}, pages 15190--15200, 2023{\natexlab{b}}.

\bibitem[Khattak et~al.(2024)Khattak, Naeem, Naseer, Van~Gool, and Tombari]{khattak2024learning}
Muhammad~Uzair Khattak, Muhammad~Ferjad Naeem, Muzammal Naseer, Luc Van~Gool, and Federico Tombari.
\newblock Learning to prompt with text only supervision for vision-language models.
\newblock \emph{arXiv preprint arXiv:2401.02418}, 2024.

\bibitem[Kundu et~al.(2021)Kundu, Das, Geem, Han, and Sarkar]{kundu2021pneumonia}
Rohit Kundu, Ritacheta Das, Zong~Woo Geem, Gi-Tae Han, and Ram Sarkar.
\newblock Pneumonia detection in chest x-ray images using an ensemble of deep learning models.
\newblock \emph{PloS one}, 16\penalty0 (9):\penalty0 e0256630, 2021.

\bibitem[LeCun et~al.(2015)LeCun, Bengio, and Hinton]{lecun2015deep}
Yann LeCun, Yoshua Bengio, and Geoffrey Hinton.
\newblock Deep learning.
\newblock \emph{nature}, 521\penalty0 (7553):\penalty0 436--444, 2015.

\bibitem[Mirza et~al.(2024)Mirza, Karlinsky, Lin, Possegger, Kozinski, Feris, and Bischof]{mirza2024LaFter}
Muhammad~Jehanzeb Mirza, Leonid Karlinsky, Wei Lin, Horst Possegger, Mateusz Kozinski, Rogerio Feris, and Horst Bischof.
\newblock Lafter: Label-free tuning of zero-shot classifier using language and unlabeled image collections.
\newblock \emph{Advances in Neural Information Processing Systems}, 36, 2024.

\bibitem[Nour et~al.(2020)Nour, Cömert, and Polat]{deepleraning}
Majid Nour, Zafer Cömert, and Kemal Polat.
\newblock A novel medical diagnosis model for covid-19 infection detection based on deep features and bayesian optimization.
\newblock \emph{Applied Soft Computing}, 97:\penalty0 106580, 2020.
\newblock ISSN 1568-4946.
\newblock \doi{https://doi.org/10.1016/j.asoc.2020.106580}.
\newblock URL \url{https://www.sciencedirect.com/science/article/pii/S1568494620305184}.

\bibitem[{\"O}zger et~al.(2020){\"O}zger, YILDIZ, Gaygisiz, Dikmen, G{\"u}lmez, Yildiz, {\c{S}}enol, Hizel, Tun{\c{c}}can, {\c{C}}a{\u{g}}lar, et~al.]{ozger2020factors}
Hasan~Sel{\c{c}}uk {\"O}zger, PINAR~AYSERT YILDIZ, {\"U}mm{\"u}g{\"u}ls{\"u}m Gaygisiz, As{\.I}ye~U{\u{g}}ra{\c{s}} Dikmen, Zehra~Demirba{\c{s}} G{\"u}lmez, Mehmet Yildiz, Esin {\c{S}}enol, Kenan Hizel, {\"O}zlem~G{\"u}zel Tun{\c{c}}can, Kayhan {\c{C}}a{\u{g}}lar, et~al.
\newblock The factors predicting pneumonia in covid-19 patients: preliminary results from auniversity hospital in turkey.
\newblock \emph{Turkish journal of medical sciences}, 50\penalty0 (8):\penalty0 1810--1816, 2020.

\bibitem[Paul et~al.(2021)Paul, Shen, Lee, Balachandar, Peng, Lu, and Summers]{Paul}
Angshuman Paul, Thomas~C. Shen, Sungwon Lee, Niranjan Balachandar, Yifan Peng, Zhiyong Lu, and Ronald~M. Summers.
\newblock Generalized zero-shot chest x-ray diagnosis through trait-guided multi-view semantic embedding with self-training.
\newblock \emph{IEEE Transactions on Medical Imaging}, 40\penalty0 (10):\penalty0 2642--2655, 2021.
\newblock \doi{10.1109/TMI.2021.3054817}.

\bibitem[Radford et~al.(2021{\natexlab{a}})Radford, Kim, Hallacy, Ramesh, Goh, Agarwal, Sastry, Askell, Mishkin, Clark, Krueger, and Sutskever]{pmlr-v139-radford21a}
Alec Radford, Jong~Wook Kim, Chris Hallacy, Aditya Ramesh, Gabriel Goh, Sandhini Agarwal, Girish Sastry, Amanda Askell, Pamela Mishkin, Jack Clark, Gretchen Krueger, and Ilya Sutskever.
\newblock Learning transferable visual models from natural language supervision.
\newblock In Marina Meila and Tong Zhang, editors, \emph{Proceedings of the 38th International Conference on Machine Learning}, volume 139 of \emph{Proceedings of Machine Learning Research}, pages 8748--8763. PMLR, 18--24 Jul 2021{\natexlab{a}}.
\newblock URL \url{https://proceedings.mlr.press/v139/radford21a.html}.

\bibitem[Radford et~al.(2021{\natexlab{b}})Radford, Kim, Hallacy, Ramesh, Goh, Agarwal, Sastry, Askell, Mishkin, Clark, et~al.]{radford2021learning}
Alec Radford, Jong~Wook Kim, Chris Hallacy, Aditya Ramesh, Gabriel Goh, Sandhini Agarwal, Girish Sastry, Amanda Askell, Pamela Mishkin, Jack Clark, et~al.
\newblock Learning transferable visual models from natural language supervision.
\newblock In \emph{International conference on machine learning}, pages 8748--8763. PMLR, 2021{\natexlab{b}}.

\bibitem[Rahman et~al.(2021)Rahman, Khandakar, Qiblawey, Tahir, Kiranyaz, Kashem, Islam, Al~Maadeed, Zughaier, Khan, et~al.]{rahman2021exploring}
Tawsifur Rahman, Amith Khandakar, Yazan Qiblawey, Anas Tahir, Serkan Kiranyaz, Saad Bin~Abul Kashem, Mohammad~Tariqul Islam, Somaya Al~Maadeed, Susu~M Zughaier, Muhammad~Salman Khan, et~al.
\newblock Exploring the effect of image enhancement techniques on covid-19 detection using chest x-ray images.
\newblock \emph{Computers in biology and medicine}, 132:\penalty0 104319, 2021.

\bibitem[Roweis and Saul(2000)]{roweis2000nonlinear}
Sam~T Roweis and Lawrence~K Saul.
\newblock Nonlinear dimensionality reduction by locally linear embedding.
\newblock \emph{science}, 290\penalty0 (5500):\penalty0 2323--2326, 2000.

\bibitem[Shu et~al.(2022)Shu, Nie, Huang, Yu, Goldstein, Anandkumar, and Xiao]{shu2022test}
Manli Shu, Weili Nie, De-An Huang, Zhiding Yu, Tom Goldstein, Anima Anandkumar, and Chaowei Xiao.
\newblock Test-time prompt tuning for zero-shot generalization in vision-language models.
\newblock \emph{Advances in Neural Information Processing Systems}, 35:\penalty0 14274--14289, 2022.

\bibitem[Souvenir et~al.(2006)Souvenir, Zhang, and Pless]{souvenir2006image}
Richard Souvenir, Qi~Zhang, and Robert Pless.
\newblock Image manifold interpolation using free-form deformations.
\newblock In \emph{Proc. IEEE International Conference on Image Processing}, pages 1437--1440, 2006.

\bibitem[Szegedy et~al.(2015)Szegedy, Liu, Jia, Sermanet, Reed, Anguelov, Erhan, Vanhoucke, and Rabinovich]{szegedy2015going}
Christian Szegedy, Wei Liu, Yangqing Jia, Pierre Sermanet, Scott Reed, Dragomir Anguelov, Dumitru Erhan, Vincent Vanhoucke, and Andrew Rabinovich.
\newblock Going deeper with convolutions.
\newblock In \emph{Proceedings of the IEEE conference on computer vision and pattern recognition}, pages 1--9, 2015.

\bibitem[Van~der Maaten and Hinton(2008)]{van2008visualizing}
Laurens Van~der Maaten and Geoffrey Hinton.
\newblock Visualizing data using t-sne.
\newblock \emph{Journal of machine learning research}, 9\penalty0 (11), 2008.

\bibitem[Wang et~al.(2022)Wang, Wu, Agarwal, and Sun]{wang2022MedCLIP}
Zifeng Wang, Zhenbang Wu, Dinesh Agarwal, and Jimeng Sun.
\newblock Medclip: Contrastive learning from unpaired medical images and text.
\newblock \emph{arXiv preprint arXiv:2210.10163}, 2022.

\bibitem[Yao et~al.(2021)Yao, Huang, Hou, Lu, Niu, Xu, Liang, Li, Jiang, and Xu]{yao2021filip}
Lewei Yao, Runhui Huang, Lu~Hou, Guansong Lu, Minzhe Niu, Hang Xu, Xiaodan Liang, Zhenguo Li, Xin Jiang, and Chunjing Xu.
\newblock Filip: Fine-grained interactive language-image pre-training.
\newblock \emph{arXiv preprint arXiv:2111.07783}, 2021.

\bibitem[Yu et~al.(2022)Yu, Wang, Vasudevan, Yeung, Seyedhosseini, and Wu]{yu2022coca}
Jiahui Yu, Zirui Wang, Vijay Vasudevan, Legg Yeung, Mojtaba Seyedhosseini, and Yonghui Wu.
\newblock Coca: Contrastive captioners are image-text foundation models.
\newblock \emph{arXiv preprint arXiv:2205.01917}, 2022.

\bibitem[Zhang et~al.(2005)Zhang, Souvenir, and Pless]{zhang2005segmentation}
Qi~Zhang, Richard Souvenir, and Robert Pless.
\newblock Segmentation informed by manifold learning.
\newblock In \emph{International Workshop on Energy Minimization Methods in Computer Vision and Pattern Recognition}, 2005.

\bibitem[Zhang et~al.(2006)Zhang, Souvenir, and Pless]{zhang2006manifold}
Qi~Zhang, Richard Souvenir, and Robert Pless.
\newblock On manifold structure of cardiac mri data: Application to segmentation.
\newblock In \emph{Proc. IEEE Conference on Computer Vision and Pattern Recognition}, pages 1092--1098, 2006.

\bibitem[Zhou et~al.(2022{\natexlab{a}})Zhou, Yang, Loy, and Liu]{zhou2022conditional}
Kaiyang Zhou, Jingkang Yang, Chen~Change Loy, and Ziwei Liu.
\newblock Conditional prompt learning for vision-language models.
\newblock In \emph{Proceedings of the IEEE/CVF conference on computer vision and pattern recognition}, pages 16816--16825, 2022{\natexlab{a}}.

\bibitem[Zhou et~al.(2022{\natexlab{b}})Zhou, Yang, Loy, and Liu]{zhou2022learning}
Kaiyang Zhou, Jingkang Yang, Chen~Change Loy, and Ziwei Liu.
\newblock Learning to prompt for vision-language models.
\newblock \emph{International Journal of Computer Vision}, 130\penalty0 (9):\penalty0 2337--2348, 2022{\natexlab{b}}.

\end{thebibliography}

\newpage
\appendix
\appendixheading{Appendix}
\section{Choice of network architecture}
\label{app_choice_architecture}
Table \ref{tab:choice_architecture} below shows ablation experiments using a multi-layer perceptron (MLP) block and a linear layer instead of a transformer block, resulting in weaker performance on medical tasks. Experiments with a ResNet model without \textsf{MedCLIP} highlight the importance of using a \textsf{MedCLIP} image encoder with a transformer for better cross-disease transferability.
\begin{table}[h]
\caption{Performance metrics for various models across different training and testing sets.}
\label{tab:choice_architecture}
\centering
\small
\begin{tabular}{lccccc}
\toprule
Model & Train & S & C & G & M \\ \midrule
\textsf{MedCLIP} + Linear & S & 42.86 & 26.25 & 62.5 & 48.28 \\ 
 & C & 42.86 & 26.25 & 62.5 & 48.28 \\ 
 & G & 42.86 & 26.25 & 62.5 & 48.28 \\ 
 & M & 42.86 & 26.25 & 62.5 & 48.28 \\ \midrule
\textsf{MedCLIP} + MLP & S & 57.14 & 73.75 & 37.5 & 51.72 \\ 
 & C & 57.14 & 73.75 & 37.5 & 51.72 \\ 
 & G & 57.14 & 73.75 & 37.5 & 51.72 \\ 
 & M & 57.14 & 73.75 & 37.5 & 51.72 \\ \midrule
ResNet & S & 89.47 & 63.85 & 43.43 & 62.07 \\ 
 & C & 63.16 & 93.93 & 76.28 & 68.97 \\ 
 & G & 48.87 & 49.67 & 37.50 & 51.72 \\ 
 & M & 60.90 & 48.30 & 57.05 & 68.97 \\ \midrule
Statistical Best & x & 54.17 & 73.75 & 62.5 & 51.72 \\ \bottomrule
\end{tabular}
\end{table}

\section{Choice of Visual Encoder}
\label{app_choice_visual_encoder}
Table \ref{tab:choice_visual_encoder} provides performance metrics for different models and architectures, comparing the TPT and LaFter models using the two different visual encoders: CLIP and MedCLIP. The metrics are reported for three categories: G, S, and M. For the TPT model, the MedCLIP encoder generally outperforms the CLIP encoder across all categories, with notable improvements in categories S and M. Specifically, MedCLIP boosts the performance in category S (54.08\%) and M (57.97\%) compared to CLIP's lower scores (50.76\% and 42.03\%, respectively). The choice of encoder impacts model performance variably across different categories. Since the overall accuracy of MedCLIP was more than CLIP, we selected MedCLIP as the visual encoder for all our experiments.
performs better comparatively.

\begin{table}[h]
\caption{Performance metrics for various models and architectures}
\label{tab:choice_visual_encoder}
\centering
\small
\begin{tabular}{llccc}
\toprule
Model   & Arch & G & S & M \\ \midrule
TPT     & \textsf{CLIP} & 72.97 & 50.76 & 42.03 \\ 
        & \textsf{MedCLIP} & 72.98 & 54.08 & 57.97 \\ \midrule
\textsf{LaFter}  & \textsf{CLIP} & 62.5 & 57.14 & 51.72 \\ 
        & \textsf{MedCLIP} & 78.53 & 42.86 & 48.28 \\ \bottomrule
\end{tabular}
\end{table}

\begin{table}[h]
\caption{Overall accuracy comparison between \textsf{CLIP} and \textsf{MedCLIP} shows that \textsf{MedCLIP} performs better comparatively.}
\label{table:overall_accuracy}
\centering
\begin{tabular}{cc}
\toprule
\textsf{CLIP} & \textsf{MedCLIP} \\ \midrule
337.12 & 354.7 \\ \bottomrule
\end{tabular}
\end{table}

\section{Adapting Baseline Methods to CXR Datasets}
\label{app_lafter_cxr}

Table \ref{table:LaFter_performance} presents the performance of the \textsf{LaFter} model on the test sets of different datasets: Shenzhen-TB (S), Covid (C), Guangzhou-PN (G), and Montgomery-TB (M). For each training dataset (denoted as Train), the table shows the model's accuracy across the different evaluation datasets. For instance, when \textsf{LaFter} is trained on Shenzhen-TB and evaluated on the same (S), it achieves an accuracy of 57.14\%. Conversely, when trained on Covid (C) and evaluated on Guangzhou-PN (G), it achieves a higher accuracy of 76.28\%. The diagonal values reported in Table \ref{tab:zero_shot_results}, which are not shown here, likely reflect the model's performance when trained and evaluated on the same dataset, providing a baseline for comparison. This table helps illustrate how well \textsf{LaFter} adapts across various CXR datasets, highlighting its strengths and weaknesses in different training and evaluation scenarios.
\begin{table}[h]
\caption{Performance was evaluated on the test set for each dataset in Section 4.1, with Table \ref{tab:zero_shot_results} reporting \textsf{LaFter}'s diagonal values.}
\label{table:LaFter_performance}
\centering
\small
\begin{tabular}{lccccc}
\toprule
Model & Train & S & C & G & M \\ 
\midrule
\textsf{LaFter} & S & 57.14 & 73.75 & 37.5 & 51.72 \\ 
 & C & 42.86 & 26.25 & 76.28 & 42.86 \\ 
 & G & 42.86 & 26.25 & 62.5 & 48.28 \\ 
 & M & 57.14 & 73.75 & 37.5 & 51.72 \\ 
 \bottomrule
\end{tabular}
\end{table}

\section{Choice of Loss Function}
\label{app_choice_loss}
Table \ref{table:ranking_loss} ranks the accuracy of various models trained on different datasets and evaluated with different loss functions. For each train-eval pair (where the training dataset and evaluation dataset are denoted as S: Shenzhen-TB, C: Covid, G: Guangzhou-PN, M: Montgomery-TB), the table lists the ranking of accuracy across four loss functions: $\gL_{\rm EC}$, $\gL_{\rm CE}$, $\gL_{\rm LC}$ with $\lambda=0$, and $\gL_{\rm LC}$ with $\lambda=0.001$. Rankings are assigned from 0 (best) to 4 (worst). For example, when trained on Shenzhen-TB and evaluated on Covid, $\gL_{\rm EC}$ is ranked 1st, whereas $\gL_{\rm CE}$ is ranked 4th in the same scenario. The table provides insights into how different loss functions perform across various dataset combinations, highlighting which configurations yield the highest accuracy.

\begin{table}[h]
\caption{Ranked accuracy for each train-eval pair across all loss functions.}
\label{table:ranking_loss}
\centering
\small
\begin{tabular}{cccccc}
\toprule
Trained on & Loss & S & C & G & M \\ \midrule
S & $\gL_{\rm EC}$ & 0 & 1 & 4 & 1 \\ 
C &  & 2 & 0 & 3 & 2 \\ 
G &  & 2 & 2 & 0 & 2.5 \\ 
M &  & 1 & 3 & 1 & 0 \\ \midrule
S & $\gL_{\rm CE}$  & 0 & 4 & 3 & 4 \\ 
C &  & 4 & 0 & 4 & 4 \\ 
G &  & 3.5 & 4 & 0 & 4 \\ 
M &  & 4 & 4 & 4 & 0 \\ \midrule
S & $\gL_{\rm LC}, \lambda=0$  & 0 & 2 & 2 & 2.5 \\ 
C &  & 1 & 0 & 2 & 1 \\ 
G &  & 3.5 & 3 & 0 & 2.5 \\ 
M &  & 2.5 & 2 & 2.5 & 0 \\ \midrule
S & $\gL_{\rm LC}, \lambda=0.001$ & 0 & 3 & 1 & 2.5 \\ 
C &  & 3 & 0 & 1 & 3 \\ 
G &  & 1 & 1 & 0 & 1 \\ 
M &  & 2.5 & 1 & 2.5 & 0 \\ \bottomrule
\end{tabular}
\end{table}

Then, we have calculated the mean average rank (MAR) for each loss function, and observe that logarithmic contrastive (LC) loss has the lowest MAR, and therefore we choose it for our framework. 

\begin{table}[h]
\caption{Mean average rank (MAR) for each loss function.}
\label{table:mean_rank}
\centering
\small
\begin{tabular}{lc}
\toprule
Loss & MAR \\ \midrule
$\gL_{\rm EC}$  & 2.04 \\ 
$\gL_{\rm CE}$  & 3.88 \\ 
$\gL_{\rm LC}, \lambda=0$  & 2.21 \\ 
$\gL_{\rm LC}, \lambda=0.001$ & 1.88 \\ \bottomrule
\end{tabular}
\end{table}

\end{document}